\documentclass[letterpaper]{article} 
\usepackage{aaai23}  
\usepackage{times}  
\usepackage{helvet}  
\usepackage{courier}  
\usepackage[hyphens]{url}  
\usepackage{graphicx} 
\usepackage{natbib}  
\usepackage{caption} 
\usepackage{algorithm}
\usepackage{algorithmic}
\usepackage{epsfig}
\usepackage{graphicx}
\usepackage{amsmath}
\usepackage{amsthm}
\usepackage{amssymb}
\usepackage{subfigure}
\usepackage{bbm}
\usepackage[table,xcdraw]{xcolor}
\usepackage[switch]{lineno}
\usepackage{rotating}
\usepackage{pifont}

\usepackage{booktabs}
\usepackage{enumitem}
\usepackage{multirow}

\usepackage{marvosym}  

\usepackage{newfloat}
\usepackage{listings}
\DeclareCaptionStyle{ruled}{labelfont=normalfont,labelsep=colon,strut=off} 
\floatstyle{ruled}
\newfloat{listing}{tb}{lst}{}
\floatname{listing}{Listing}
\title{DesNet: Decomposed Scale-Consistent Network for\\Unsupervised Depth Completion}

\author{
    Zhiqiang Yan, Kun Wang, Xiang Li, Zhenyu Zhang, Jun Li\textsuperscript{\Letter}, Jian Yang\textsuperscript{\Letter}
}

\affiliations{
    PCA Lab, Nanjing University of Science and Technology, Jiangsu Province, China\\
    \{Yanzq,kunwang,xiang.li.implus,junli,csjyang\}@njust.edu.cn, zhangjesse@foxmail.com
}

\usepackage{bibentry}

\begin{document}

\maketitle

\begin{abstract}
Unsupervised depth completion aims to recover dense depth from the sparse one without using the ground-truth annotation. Although depth measurement obtained from LiDAR is usually sparse, it contains valid and real distance information, \emph{i.e.}, scale-consistent absolute depth values. Meanwhile, scale-agnostic counterparts seek to estimate relative depth and have achieved impressive performance. To leverage both the inherent characteristics, we thus suggest to model scale-consistent depth upon unsupervised scale-agnostic frameworks. Specifically, we propose the \emph{decomposed scale-consistent learning} (DSCL) strategy, which disintegrates the absolute depth into relative depth prediction and global scale estimation, contributing to individual learning benefits. But unfortunately, most existing unsupervised scale-agnostic frameworks heavily suffer from depth holes due to the extremely sparse depth input and weak supervised signal. To tackle this issue, we introduce the \emph{global depth guidance} (GDG) module, which attentively propagates dense depth reference into the sparse target via novel dense-to-sparse attention. Extensive experiments show the superiority of our method on outdoor KITTI benchmark, ranking 1st and outperforming the best KBNet more than $12\%$ in RMSE. In addition, our approach achieves state-of-the-art performance on indoor NYUv2 dataset.
\end{abstract}

\begin{figure}[t]
 \centering
 \includegraphics[width=0.95\columnwidth]{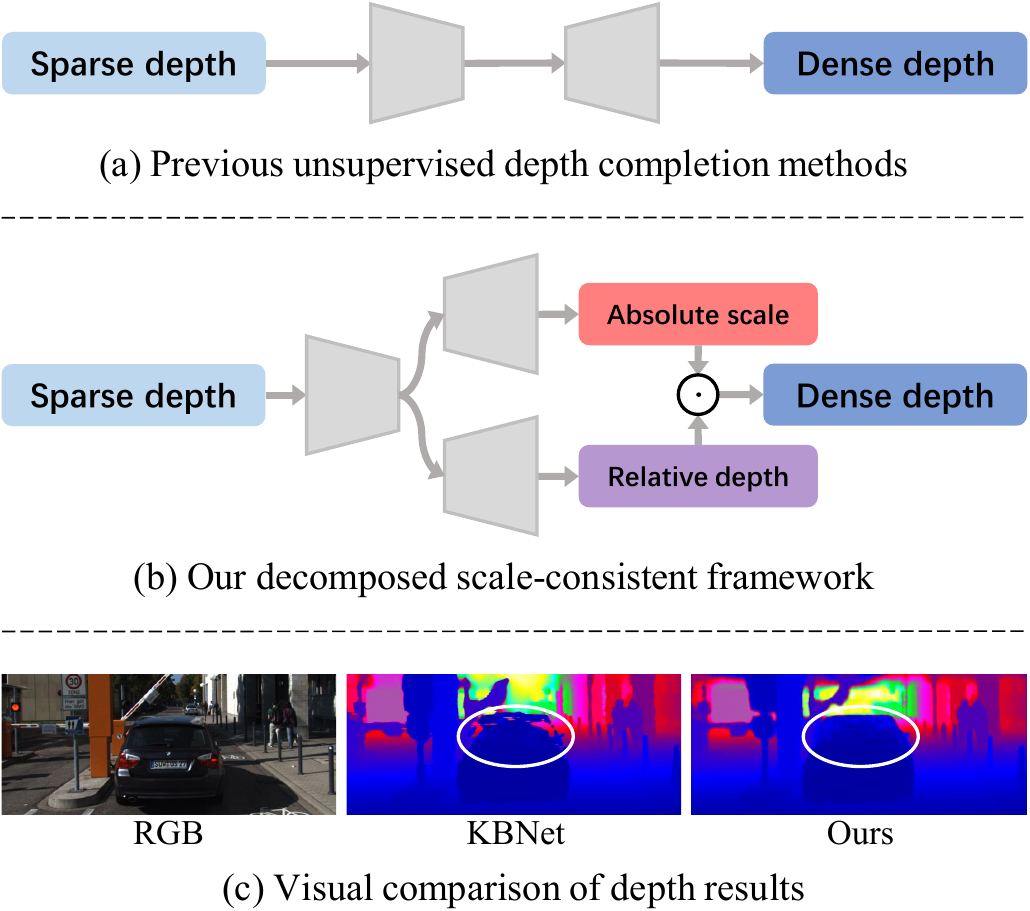}\\
 \caption{(a) Previous methods directly predict the absolute dense depth, whilst (b) our approach decomposes it into relative depth prediction and absolute scale estimation, contributing to (c) not only higher accuracy but denser depth than the excellent KBNet~\cite{wong2021unsupervised}.}\label{comparison_with_KBNet}
\end{figure}

\section{Introduction}
Depth completion, converting sparse depth to the dense one with or without the help of the corresponding image, is an indispensable part of many computer vision applications, \emph{e.g.}, autonomous driving~\cite{godard2019digging,yan2021rignet}, augmented reality~\cite{zhong2019deep,yan2022learning}, and 3D scene reconstruction~\cite{zhang2019pattern,yan2022multi}. In these scenarios, dense ground-truth depth annotations are usually expensive and hard to obtain for supervised depth completion while sparse depth maps can be easily measured by depth sensors (\emph{e.g.}, LiDAR). Hence, plenty of unsupervised approaches~\cite{2020Dense,wong2021unsupervised} have been proposed to reduce the high cost since S2D~\cite{ma2018self} establishes the first unsupervised framework for depth completion. In general, \emph{for one thing}, these methods are supervised by the sparse depth input and photometric reconstruction, which directly output scale-consistent absolute depth upon the real scale information in sparse depth. \emph{For another thing}, unsupervised depth estimation counterparts (only with color images as input) often suffer from scale ambiguity issues~\cite{bian2019unsupervised}, which could only predict relative depth. However, in recent months, the counterparts have shown promising prospect that contributes to high depth accuracy~\cite{petrovai2022exploiting,he_ra_depth}. \emph{These analyses motivate us to leverage the scale information in sparse depth and the high accuracy of scale-agnostic counterparts for unsupervised depth completion}.

Consequently, in this paper we attempt to explore a new solution to the unsupervised depth completion task, \emph{i.e.}, the \emph{decomposed scale-consistent learning} (DSCL) strategy. As illustrated in Fig.~\ref{comparison_with_KBNet}(a) and (b), totally different from previous approaches that directly estimates scale-consistent depth, our DSCL first learns scale-agnostic relative depth \& real global scale factor and then outputs the absolute target. Theoretically, we prove that such individual learning is more effective and conduces to better depth results. However, as shown in Fig.~\ref{comparison_with_KBNet}(c), since the depth input is extremely sparse (about 5\% valid pixels) and the supervised signal is very weak, the mainstream unsupervised depth completion methods, \emph{e.g.}, the excellent Oral KBNet~\cite{wong2021unsupervised} in \emph{ICCV}, heavily suffers from depth holes which thus lead to large inaccuracy as well as unreasonable visual effect.

To tackle this problem, we present the \emph{global depth guidance} (GDG) module. It first produces coarse but dense depth reference by morphological dilation technology~\cite{jackway1996scale}. Then a novel dense-to-sparse attention is designed to effectively propagate the dense depth reference into the sparse target. Concretely, this attention urges to learn the residual of sparse target by capturing non-local correlations between the sparse-modal and dense-modal features, contributing to satisfactory compensation for depth holes. In addition, a fast version of the dense-to-sparse attention is further proposed to realize high efficiency, which largely reduces the complexity from quadratic to linear.

In summary, our main contributions are listed as follows:
\begin{itemize}
    \item We introduce a new solution to the unsupervised depth completion task, \emph{i.e.}, the decomposed scale-consistent learning framework that disintegrates the absolute depth into relative depth prediction and global scale estimation.
    \item A global depth guidance module is proposed to deal with the issue of depth holes, including a dense-to-sparse attention that learns long-range correlations between sparse-modal and dense-modal features.
    \item Extensive experiments verify the effectiveness of our approach, which achieves the state-of-the-art performance on both outdoor KITTI and indoor NYUv2 benchmarks.
\end{itemize}

\section{Related Work}
\textbf{Depth Completion.} The basic task of depth completion has attracted much public attention since the work~\cite{Uhrig2017THREEDV} first proposes sparsity invariant CNNs to fill missing depth values. In general, depth completion can be broadly categorized into supervised and unsupervised learning. \emph{For supervised learning}, existing methods mainly take as input a single sparse depth or multiple sensor information~\cite{yan2021rignet}, 
which has greatly promoted the development of the depth completion task. For example, \cite{ma2018self} utilize an hourglass network to recover dense depth based on a single sparse depth. \cite{2020FromLu} employ sparse depth as the only input and further use the corresponding color image as an auxiliary supervisory signal to provide semantic information. CSPN~\cite{2018Learning}, NLSPN~\cite{park2020nonlocal}, and DySPN~\cite{lin2022dynamic} refine coarse depth by learning affinity matrix with spatial propagation network based on RGB-D pair. 
GuideNet~\cite{tang2020learning} and RigNet~\cite{yan2021rignet} present image-guided methods to benefit depth completion. DeepLiDAR~\cite{Qiu_2019_CVPR} jointly uses color image, surface normal, and sparse depth for more precise depth recovery. To robustly predict dense depth, uncertainty estimation~\cite{vangansbeke2019,Qu_2021_ICCV,zhu2021robust} is introduced to tackle outlier and obscure. \emph{For unsupervised learning}, there are lots of works~\cite{ma2018self,2020Dense,wong2021unsupervised} focusing on simultaneously completing sparse depth and reducing expensive ground-truth costs. For example, \cite{ma2018self} build a solid framework to concurrently deal with supervised, self-supervised, and unsupervised depth completion. \cite{wong2021learning} utilize synthetic data for further improvement. Recently, KBNet~\cite{wong2021unsupervised} proposes calibrated backprojection that significantly facilitates the unsupervised depth completion task. However, most of these methods directly predict absolute depth and often suffer from depth holes near key elements for self-driving, \emph{e.g.}, cars. Different from them, we present a new unsupervised solution that decomposes the absolute depth into relative depth prediction and global scale estimation.

\noindent \textbf{Depth Estimation.} The tasks of depth completion and depth estimation are closely relevant. The major difference between them is that the former has additional sparse depth information as input while the latter does not. Research on depth estimation can trace further back to the early method~\cite{saxena2005learning}. Since then, many \emph{supervised approaches}~\cite{roy2016monocular,lee2019big,zhang2019pattern} have been proposed which greatly promote the development of this domain. Furthermore, \cite{zhou2017unsupervised} propose the first \emph{unsupervised depth estimation} system, which takes view synthesis as the supervisory signal during end-to-end training. This work has laid a good foundation for subsequent research. After that, various unsupervised works~\cite{godard2019digging,bian2019unsupervised,shu2020featdepth,zhao2020towards,he_ra_depth} are burgeoning over the past three years. Although these methods suffer from the scale ambiguity issues all along, they have achieved promising performance with high depth accuracy. Therefore, it becomes possible to explore a new solution to unsupervised depth completion task based on these scale-agnostic frameworks and real scale information in sparse depth input.

\noindent \textbf{Scale Decomposition in Depth.} There are some depth estimation works related to scale decomposition, including scale-agnostic and scale-consistent categories. \emph{For scale-agnostic methods},~\cite{eigen2014depth} present to learn relative depth that is normalized to (0, 1) to tackle the ambiguous scale issue. Meanwhile,~\cite{xian2020structure} propose pair-wise ranking loss guided by structure to improve the quality of depth prediction. Further,~\cite{ranftl2020towards} ameliorates the generalization capability on multiple datasets with different scales.~\cite{wang2020sdc} build a new framework by depth and scale decomposition in a supervised manner. \emph{For scale-consistent approaches}, many works~\cite{chen2019self,wang2021can} utilize geometric consistency in 2D and 3D spaces to model consistent scales. In addition,~\cite{guizilini20203d} takes camera velocity as extra supervised signal to mitigate the scale-agnostic issue. Moreover,~\cite{tiwari2020pseudo} introduce bundle-adjusted 3D scene structures to benefit their depth prediction network. Different from them, to seek a new solution to the unsupervised depth completion task, we predict scale-consistent depth via scale-agnostic basis with the help of real scale information in sparse depth input.

 \begin{figure*}[t]
  \centering
  \includegraphics[width=1.5\columnwidth]{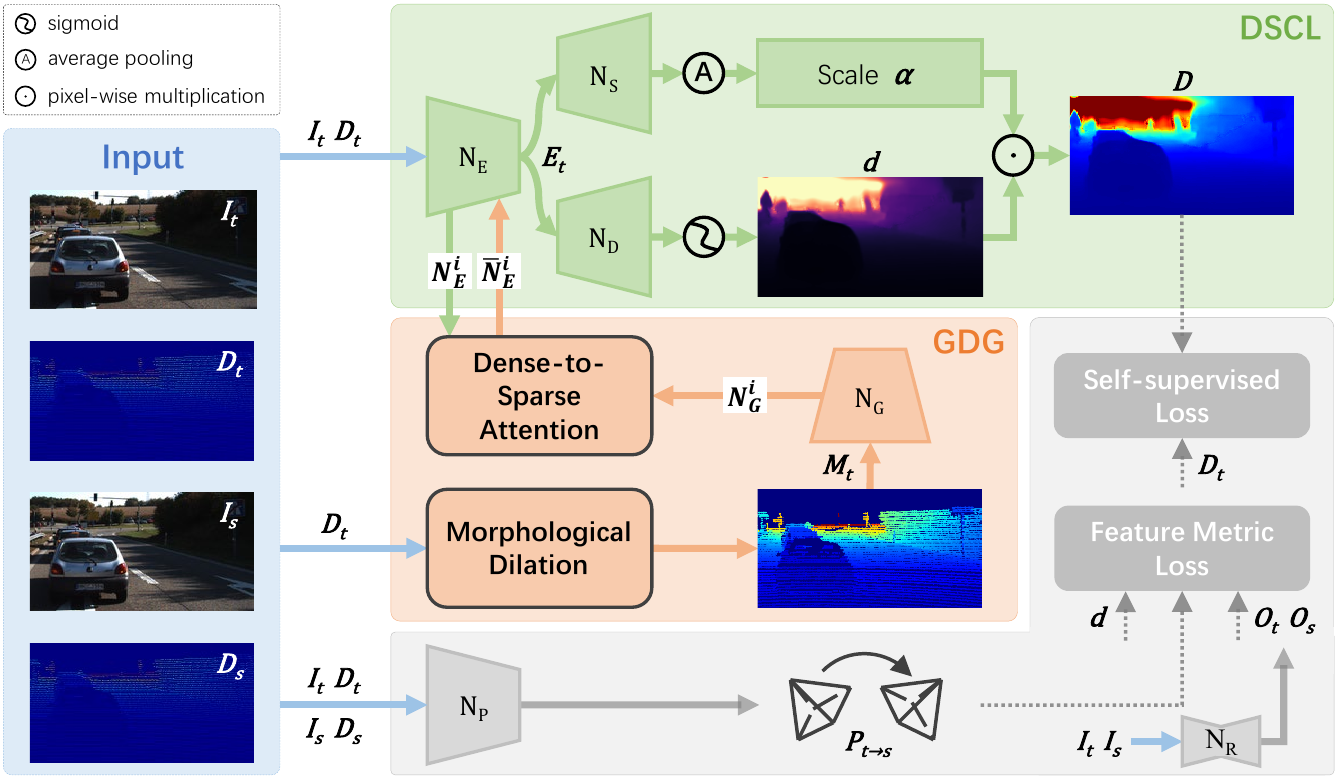}\\
  \caption{Overview of our unsupervised depth completion framework, where the \emph{decomposed scale-consistent learning} (DSCL) is designed to disentangle the absolute depth into relative depth prediction and absolute scale estimation. Meanwhile, the \emph{global depth guidance} (GDG) is introduced to provide the depth network in DSCL with dense depth reference.}\label{Fig.2}
\end{figure*}

\section{Decomposed Scale-Consistent Network}
In this section, we first introduce some prior knowledge in Sec.~\ref{prioi_knowledge} and the overall network architecture in Sec.~\ref{overall_arc}. Then we elaborate on the two key designs of our method in Secs.~\ref{DSCL} and~\ref{GDG}. For simplicity, the proposed \textbf{De}composed \textbf{S}cale-Consistent \textbf{Net}work is termed as \textbf{DesNet}.

\subsection{Prior Knowledge}\label{prioi_knowledge}
\textbf{Camera model.} The optical camera projects a 3D point $ Q=\left(X,Y,Z \right)$ to a 2D pixel $q=\left(u,v \right)$ by
\begin{equation}\label{e_projection}
\begin{split}
    \pi \left( Q \right)=\left( {{h}_{x}}\frac{X}{Z}+{{c}_{x}},{{h}_{y}}\frac{Y}{Z}+{{c}_{y}} \right),
\end{split}
\end{equation}
where $\pi$ represents camera operator, $\left(h_x,h_y,c_x,c_y \right)$ are the optical camera intrinsic parameters. Given depth $\boldsymbol d$ and its pixel $d_q$, the corresponding backprojection process is 
\begin{equation}\label{e_backprojection}
\begin{split}
    {{\pi }^{-1}}\left( q, d_q \right)=d_q{{\left( \frac{x-{{c}_{x}}}{{{h}_{x}}},\frac{y-{{c}_{y}}}{{{h}_{y}}},1 \right)}^T}.
\end{split}
\end{equation}

\noindent \textbf{Ego-motion.} Ego-motion can be modeled by transformation $G$. Warping function $\omega$ maps a pixel $q$ in one frame to another frame, obtaining the corresponding pixel $\widetilde{q}$. It can be described as
\begin{equation}\label{e_motion}
\begin{split}
    \widetilde{q}=\omega \left( q,{{d}_{q}},G \right)=\pi \left( G\cdot {{\pi }^{-1}}\left( q,{{d}_{q}} \right) \right).
\end{split}
\end{equation}

\subsection{Overview of Network Architecture}\label{overall_arc}
The whole framework of our method is shown in Fig.~\ref{Fig.2}. Without loss of generality, we define monocular RGB-D target frames $\boldsymbol I_t$ (color image), $\boldsymbol D_t$ (sparse depth), and source frames $\boldsymbol I_s$, $\boldsymbol D_s$ as input. $\boldsymbol O_t$ and $\boldsymbol O_s$ are the color space features of $\boldsymbol I_t$ and $\boldsymbol I_s$, which are encoded by a shared image reconstruction network $N_R$~\cite{shu2020featdepth}.

For decomposed scale-consistent learning (DSCL), we first predict the relative depth $\boldsymbol d$ by the depth network $N_D$ with sigmoid mapping. Concurrently, the global scale factor $\alpha$ is estimated by the scale network $N_S$ with average pooling based on $\boldsymbol E_{t}$ produced from the shared encoder $N_E$, where a dense-to-sparse attention is proposed to propagate dense depth reference. Finally, we multiply $\boldsymbol d$ by $\alpha$ to generate the scale-consistent absolute depth prediction $\boldsymbol D$.

For global depth guidance (GDG), we first transform $\boldsymbol D_t$ to a denser depth $\boldsymbol M_t$ by morphological dilation technology. Then we employ the guidance network $N_G$ to map $\boldsymbol M_t$ into feature space. The features in $i$th layer of $N_G$ and $N_E$ are $\boldsymbol N^{i}_{G}$ and $\boldsymbol N^{i}_{E}$ respectively, both of which are input into the dense-to-sparse attention module to update $\boldsymbol N^{i}_{E}$ to $\bar{\boldsymbol N}_{E}^{i}$. 

For supervised signal, we first use $\boldsymbol D_t$ as the primary supervision of the absolute depth prediction $\boldsymbol D$. Then following~\cite{shu2020featdepth}, we employ the feature metric loss as the auxiliary supervision, which inputs $\boldsymbol O_t$, $\boldsymbol O_s$, and the pose $\boldsymbol P_{t\to s}$ that is predicted by the pose network $N_P$. Next, we warp $\boldsymbol O_t$ to $\boldsymbol {O}_{t\to s}$ with $\boldsymbol d$ and $\boldsymbol P_{t\to s}$ used. A cross-view reconstruction loss is thus applied between $\boldsymbol {O}_{s}$ and $\boldsymbol {O}_{t\to s}$.

It is worth noting that, when testing, our model only needs $\boldsymbol I_t$ and $\boldsymbol D_t$ to generate the final depth prediction $\boldsymbol D$.

\begin{figure*}[t]
 \centering
 \includegraphics[width=1.8\columnwidth]{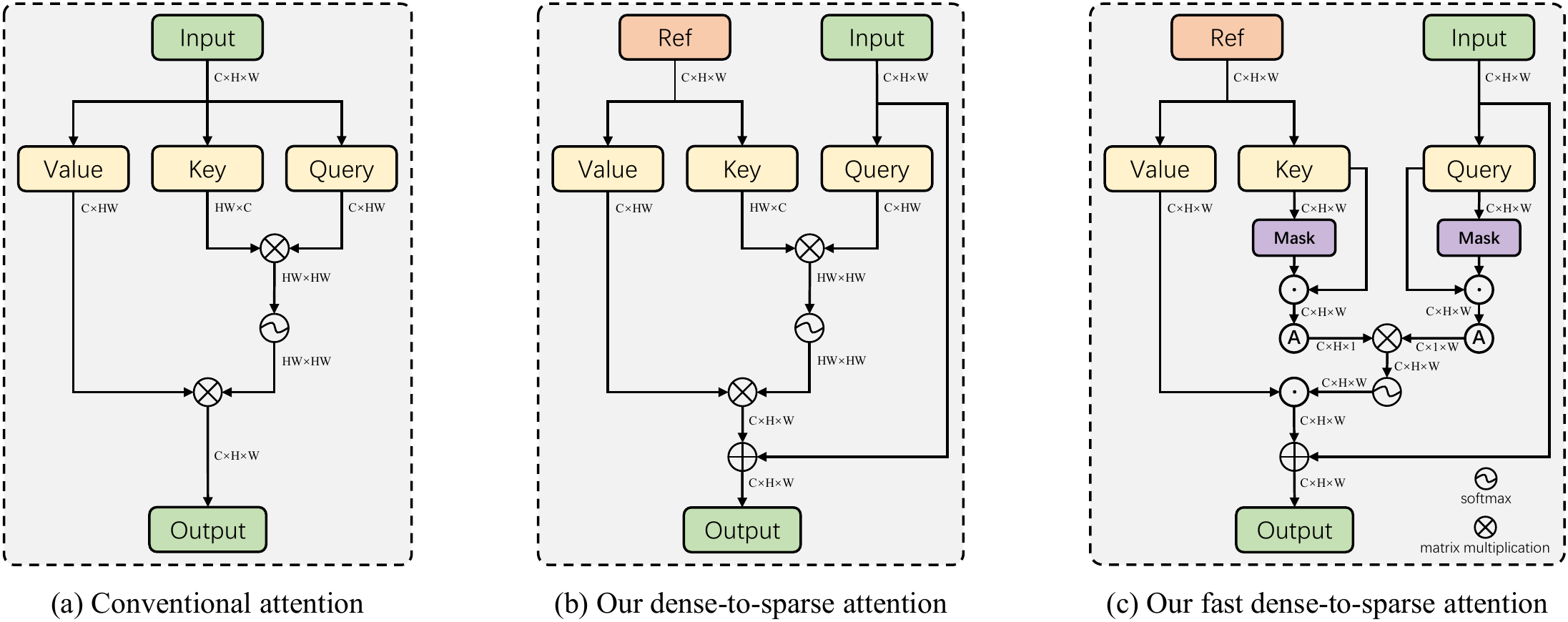}\\
 \caption{Comparison of conventional attention~\cite{dosovitskiy2020image} and our proposed dense-to-sparse attention.}\label{dense-to-sparse}
\end{figure*}

\subsection{Decomposed Scale-consistent Learning}\label{DSCL}
Sparse depth maps possess actual and precise depth values, which contain real scale information~\cite{Uhrig2017THREEDV} that can provide significant guidance for the unsupervised setting. On the other hand, existing scale-agnostic unsupervised depth estimation methods~\cite{godard2019digging} have achieved impressive performance, especially the high accuracy. Accordingly, we attempt to leverage the real scale information in sparse depth and the high accuracy of scale-agnostic counterparts for the unsupervised depth completion task. Hence, as shown in the green box of Fig.~\ref{Fig.2}, we propose the decomposed scale-consistent learning strategy.

Specifically, we decompose the absolute depth $\boldsymbol D$ into relative depth $\boldsymbol d$ prediction and global scale $\alpha $ estimation. The depth network $N_D$ produces $\boldsymbol d\in (0,1]$ with sigmoid function used, which is different from existing unsupervised depth completion methods \cite{ma2018self,2020Dense,wong2021unsupervised} that directly generate 0$\sim$80$m$ absolute depth. Meanwhile, we leverage the scale network $N_S$ to estimate the real scale $\alpha$ with $\boldsymbol E_t$ as input, which is the final layer feature of the shared ResNet-18 encoder $N_E$. The DSCL is defined as
\begin{equation}\label{e_DSCL}
\begin{split}
    &\boldsymbol D=\alpha \cdot \boldsymbol d, \\
    &\alpha=N_{S}\left (\boldsymbol E_{t} \right),
\end{split}
\end{equation}
where $N_{S} \left( \cdot \right)$ refers to the corresponding function of depth network $N_S$. Other functions have the same definition next.

Here, we provide a theory to show that our DSCL can help the network to recover better depth using $\mathcal{L}_2$ loss.

\newtheorem{thm}{Theorem}[section]
\begin{thm}\label{thm1}
Given a sparse depth $\boldsymbol D_{t}$ and a depth prediction $\boldsymbol D$ with network parameters $w$, if $w$ does not satisfy $\boldsymbol D=0$, then there exists a scale factor $\alpha\neq 0$ such that $\sum_{\Omega}(\boldsymbol D_t-\boldsymbol D)^2\geq\sum_{\Omega}(\boldsymbol D_t-\alpha \boldsymbol D)^2$, where $\Omega\neq \emptyset $ is the index set of pixel location using the supervision.
\end{thm}

\begin{proof}
For simplicity, we consider only one pixel $q$ in the proof process. For the traditional self-supervision, the loss is $\min_{w}(\boldsymbol D_t^{q}-\boldsymbol D^q )^2$. For DSCL, we introduce a scale factor $\alpha$ to the loss, and have a new loss $\min_{w}(\boldsymbol D_t^{q}-\alpha \boldsymbol D^q )^2$. It is easy to prove that if $\boldsymbol D^q\neq 0$, then there has a $\alpha$ such that $(\boldsymbol D_t^{q}-\boldsymbol D^q )^2\geq(\boldsymbol D_t^{q}-\alpha \boldsymbol D^q )^2$. Furthermore, $\alpha$ has a closed form $\alpha=\boldsymbol D_t^{q}/\boldsymbol D^{q}$. When $\boldsymbol D^{q}=\boldsymbol D_t^{q}$, $\alpha=1$. 
\end{proof}

Theorem \ref{thm1} shows that $\alpha \boldsymbol D$ is closer to the supervised signal $\boldsymbol D_t$ than the original $\boldsymbol D$. In fact, it is impossible for the network to predict $\boldsymbol D$ that equals to $\boldsymbol D_t$. Thus, $\alpha\neq 1$ drives that $\alpha \boldsymbol D$ has better approximation than $\boldsymbol D$. It reveals that our network has a strong practical significance to predict the depth. Furthermore, we employ a scale network to learn the scale factor from the data and map the absolute $\boldsymbol D$ to a relative $\boldsymbol d$, which is easier for networks to optimize.

\begin{figure}[t]
 \centering
 \includegraphics[width=0.572\columnwidth]{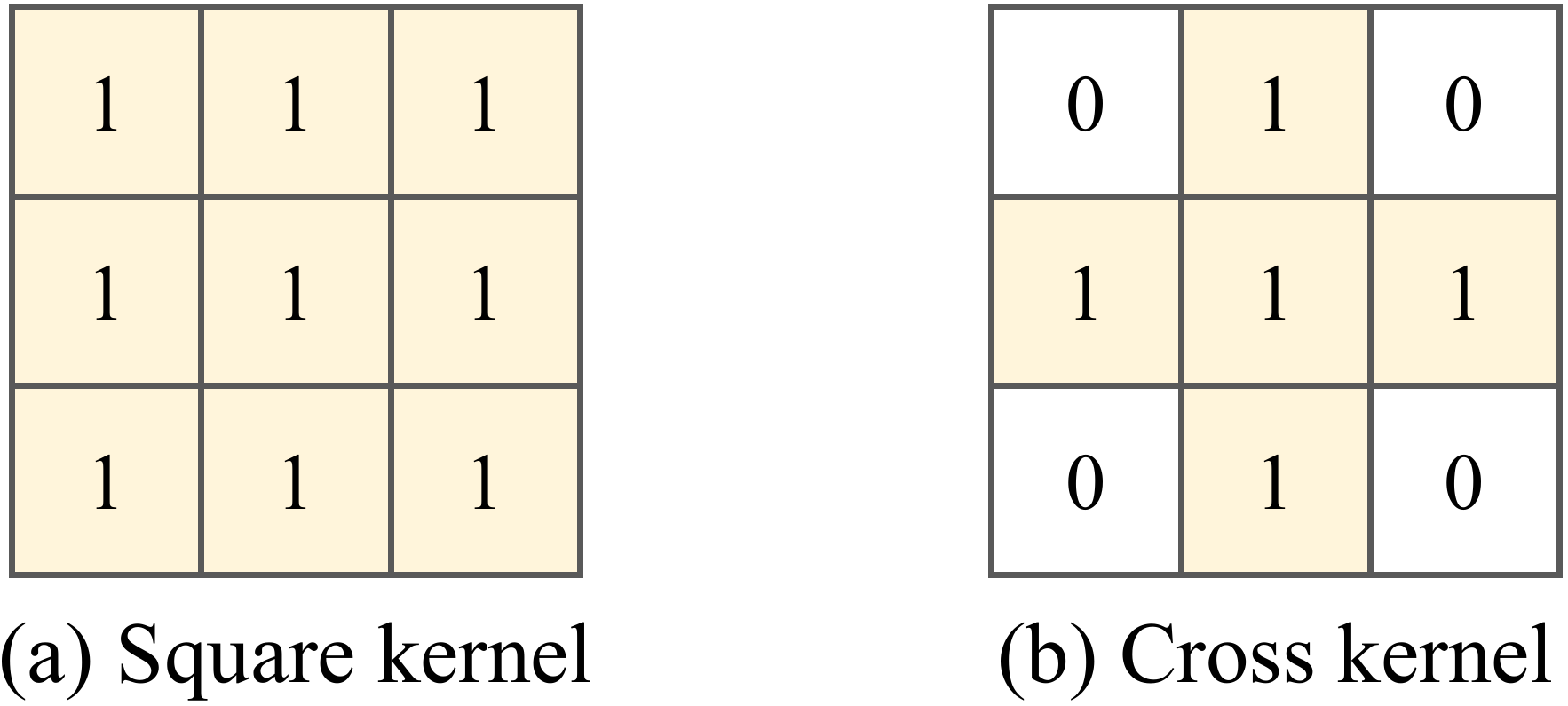}\\
 \caption{Different kernels with square and cross shapes.}\label{dilation kernel}
\end{figure}

\subsection{Global Depth Guidance}\label{GDG}
Existing unsupervised depth completion works~\cite{ma2018self,shivakumar2019dfusenet,2020Dense} have shown promising performance. However, most of them~\cite{wong2020unsupervised,wong2021learning,wong2021adaptive,wong2021unsupervised} suffer from depth holes in depth results. When we try to mitigate this issue, there are \emph{two problems} impeding us. 

\emph{The first problem is that,} compared with related image processing works whose inputs are totally complete, the depth input $\boldsymbol D_t$ so sparse that it cannot provide dense information, which is the main cause of depth holes. To alleviate this issue, we take advantage of morphological dilation technology~\cite{jackway1996scale} with different kernels (Fig.~\ref{dilation kernel}) to provide coarse but much denser depth $\boldsymbol M_t$.

\emph{The second problem is that}, how to fuse the dense depth reference and the sparse depth input? Inspired by the conventional attention~\cite{dosovitskiy2020image}, as illustrated in Fig.~\ref{dense-to-sparse}, we propose the dense-to-sparse attention. We first use $N_E$ and $N_G$ to map the RGB-D input ($\boldsymbol I_t$, $\boldsymbol D_t$) and the dense $\boldsymbol M_t$ into feature spaces, generating $\boldsymbol N^{i}_{E}$ and $\boldsymbol N^{i}_{G}$ in the $i$th network layer. Then, the dense-to-sparse attention propagates the dense depth reference $\boldsymbol N^{i}_{G}$ into the sparse depth feature $\boldsymbol N^{i}_{E}$, and thus obtaining the updated feature $\bar{\boldsymbol N}_{E}^{i}$. In addition, it is well-known that such attention has very high complexity even though its strong performance. Therefore, we design a fast version to deal with this issue.

The above process can be described as
\begin{equation}\label{e_GDG}
\begin{split}
    &\bar{\boldsymbol N}_{E}^{i}=f_{att}\left (\boldsymbol N^{i}_{E}, \boldsymbol N^{i}_{G} \right), \\
    &\boldsymbol N_{E}^{i}=N_{E}\left (\boldsymbol I_{t},\boldsymbol D_{t} \right), \\
    &\boldsymbol N_{G}^{i}=N_{G}\left (f_{dil}\left (\boldsymbol D_{t} \right) \right),
\end{split}
\end{equation}
where $f_{att} \left( \cdot \right)$ denotes the densn-to-sparse attention function and $f_{dil} \left( \cdot \right)$ is the morphological dilation function.

Different from the conventional attention~\cite{dosovitskiy2020image} in Fig.~\ref{dense-to-sparse}(a) that takes single-modal input as value, key, and query, our attention in Fig.~\ref{dense-to-sparse}(b) employs sparse depth feature as query, and dense depth reference as value and key. Also unlike the multi-modal attention~\cite{rho2022guideformer,li2022deepfusion} which use absolutely different-modal (RGB and LiDAR) data, our attention only leverages LiDAR data, \emph{i.e.}, sparse-modal and dense-modal depth. Additionally, to reduce the high complexity of (a) and (b) that equals to $(HW)^2$ when calculating the correlation between key and query, we design a fast version of the dense-to-sparse attention in Fig.~\ref{dense-to-sparse}(c), where \emph{two key steps} are conducted. \emph{One key step} is the binary mask, aiming to reduce redundancy existed in the long-range correlation since not every pixel in key is always related to that in query. \emph{Another key step} is the strip average pooling, compressing the $HW$ key and $HW$ query into $H$ and $W$, respectively. Then we multiply the compressed key by the compressed query to obtain the correlation matrix whose complexity is only $HW$, much smaller than $(HW)^2$ of the conventional attention.

\subsection{Loss Function}
The total loss function contains a cross-view reconstruction loss $\mathcal{L}_{t \to s}$, a single-view reconstruction loss $\mathcal{L}_{si}$, and a self-supervised loss $\mathcal{L}_2$ to predict the final depth result $\boldsymbol D$.

\textbf{Cross-view reconstruction loss.} 
Based on the geometry model defined in Eq.~\ref{e_projection}, the source frame $\boldsymbol O_s$ can be rebuilt from target frame $\boldsymbol O_t$ via ${\boldsymbol {\widetilde{O}}_{t\to s}}(q)={\boldsymbol {O}_{t}}(\widetilde{q})$. Then, the cross-view reconstruction loss is
\begin{equation}\label{e_cross-view-2}
\begin{split}
    \mathcal{L}_{t \to s}=\sum\nolimits_{q}{{\left|{\boldsymbol {O}_{t}}\left( \widetilde{q} \right)-{\boldsymbol {O}_{s}}\left( q \right) \right|}}.
\end{split}
\end{equation}

\textbf{Single-view reconstruction loss.}
Given color image $\boldsymbol I$, the shared reconstruction network $N_R$ maps the feature representation $\boldsymbol O$. The single-view reconstruction loss is
\begin{equation}\label{e_single-view}
\begin{split}
    \mathcal{L}_{si}=&\sum\nolimits_{q}{{\left|\boldsymbol I(q)-\boldsymbol O(q) \right|}}+\alpha \sum\nolimits_{q}{{\left| {{\nabla }^{2}}\boldsymbol O(q) \right|}}\\
    &+\beta \left( -\sum\nolimits_{q}{{{e}^{-{{\left| {{\nabla }^{1}}\boldsymbol I(q) \right|}}}}\centerdot {{\left| {{\nabla }^{1}}\boldsymbol O(q) \right|}}} \right),
\end{split}
\end{equation}
where $\alpha=\beta=1e-3$, $\nabla^1$ and $\nabla^2$ denote the first-order derivative and the the second-order derivative, respectively.

\textbf{Decomposed scale-consistent learning loss.}
The final depth $\boldsymbol D$ is supervised by $\mathcal{L}_2$ loss, which can be defined as
\begin{equation}\label{e_l2}
\begin{split}
    \mathcal{L}_{2}=\sum\nolimits_{\text{q}}{{{\left| {\boldsymbol {D}_{gt}}\left( q \right)-\boldsymbol D\left( q \right) \right|}^{2}}}.
\end{split}
\end{equation}

\textbf{Total loss.}
Finally, the total loss function is written as
\begin{equation}\label{e3}
\begin{split}
    \mathcal{L}=\mathcal{L}_{t \to s}+\mathcal{L}_{si}+\gamma \mathcal{L}_2,
\end{split}
\end{equation}
where $\gamma$ is set to 1 during training. Please refer to \cite{shu2020featdepth} for more details about $\mathcal{L}_{t \to s}$ and $\mathcal{L}_{si}$ loss functions.

\section{Experiment}
Here, we first introduce related datasets and implementation details. Then we conduct ablation studies to verify the effectiveness of our method. Finally, we compare our method against other state-of-the-art approaches. Following KITTI benchmark, RMSE ($mm$) is selected as the \emph{primary metric}.

\subsection{Datasets and Implementation Details}

\textbf{KITTI benchmark}~\cite{Uhrig2017THREEDV} consists of \textbf{86,898} RGB-D pairs for training, 7,000 for validating, and another 1,000 for testing. The official 1,000 validation images are used during training while the remaining images are ignored. Following GuideNet~\cite{tang2020learning}, RGB-D pairs are bottom center cropped from $1216 \times352$ to $1216 \times256$, as there are no valid LiDAR values near top 100 pixels.

\noindent \textbf{NYUv2 dataset}~\cite{silberman2012indoor} contains 464 RGB-D indoor scenes with $640 \times480$ resolution. Following KBNet~\cite{wong2021unsupervised}, we train our model on 46K frames and test on the official test set with 654 images. The sparse depth input is artificially produced by sampling about 1500 valid points from the ground-truth (GT) depth.

\noindent \textbf{Implementation Details.} We implement DesNet on Pytorch with 2 TITAN RTX GPUs. We train it for 25 epochs with Adam \cite{Kingma2014Adam} optimizer. The learning rate is gradually warmed up to $10^{-4}$ in 3 steps, where each step increases learning rate by $10^{-4}/3$ in 500 iterations. After that, the learning rate $10^{-4}$ is used for the first 20 epochs and is reduced to half at the beginning of the 20th epoch.

\subsection{Ablation Studies}
This subsection verifies the effectiveness of DesNet, including the decomposed scale-consistent learning (DSCL) and global depth guidance (GDG), on \textbf{KITTI validation split}. \emph{Gray background} in Tabs.~\ref{t_DesNet}-~\ref{t_GDG_attention} refers to our \emph{default setting}.

\textbf{DesNet.} As reported in Tab.~\ref{t_DesNet}, the baseline DesNet-i directly predicts absolute depth that is supervised by the total loss in Eq.~\ref{e3} without using scale decomposition, \emph{i.e.}, its depth network is an UNet which consists of our $N_{E}$ and $N_D$ without using sigmoid mapping. \textbf{(1)} When employing our DSCL strategy (DesNet-ii), we observe that all four evaluation metrics are consistently improved, \emph{e.g.}, RMSE is reduced by $91.9mm$ and MAE by $24.1mm$. As shown in the 3rd and 4th columns of Fig.~\ref{ab_vis}, DSCL notably contributes to sharper depth details and more complete object shapes. These numerical and visual results provide evidence that our DSCL, disintegrating the learning of absolute depth into \emph{explicit} relative depth prediction and scale estimation, assuredly reduces the learning difficulty and brings individual learning benefits (Theorem~\ref{thm1}). \textbf{(2)} When conducting our GDG module (DesNet-iii), the model performance is significantly improved. The RMSE, MAE, iRMSE, and iMAE outperform the baseline by $114.6mm$, $27.9mm$, $0.5{km}^{-1}$, and $0.3{km}^{-1}$, respectively. As illustrated in the 2nd and 3rd columns of Fig.~\ref{ab_vis}, GDG remarkably corrects the wrong depth values near cars and can compensate depth holes well where even if the GT depth annotations have no valid pixels. These evidences indicate that GDG is able to provide valid


\begin{figure}[t]
\centering
\includegraphics[width=0.95\columnwidth]{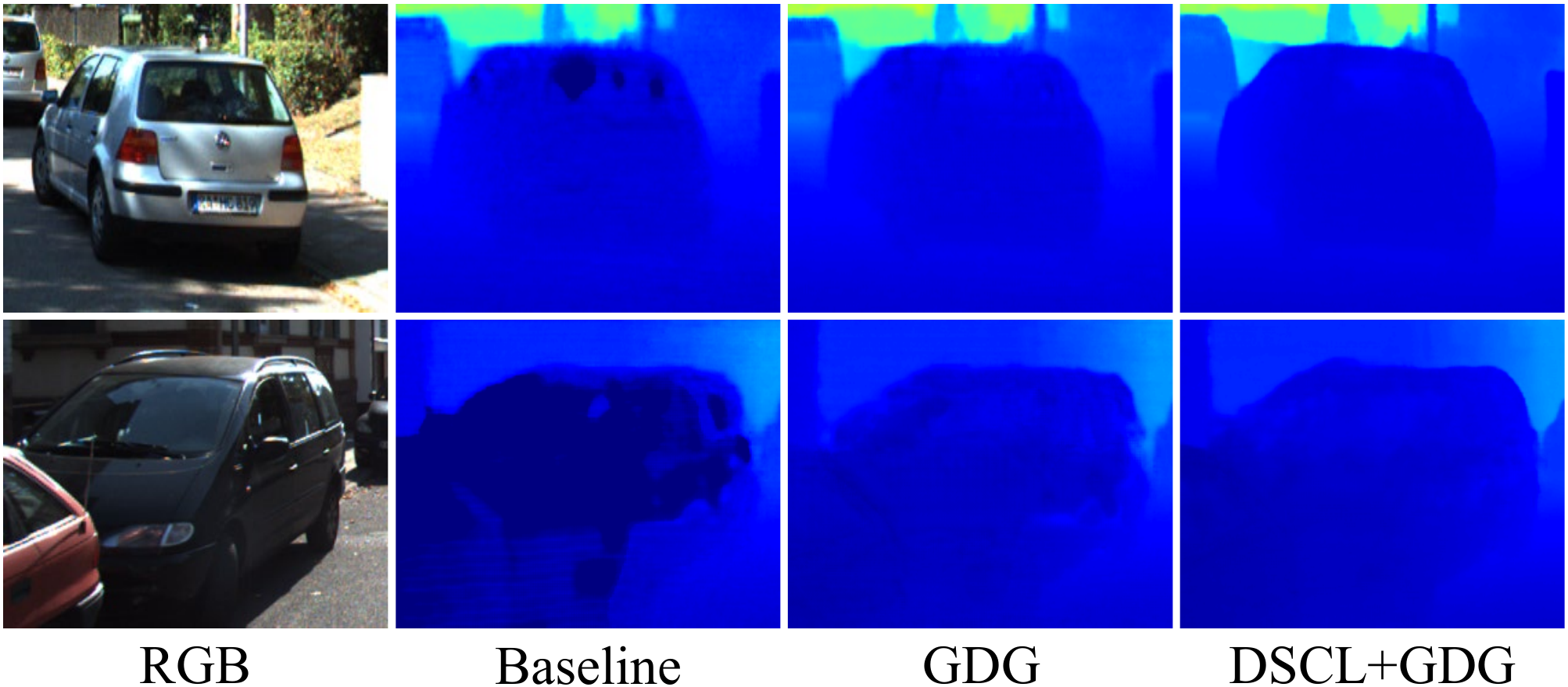}
\caption{Visual comparison of ablation studies in Tab.~\ref{t_DesNet}.}
\label{ab_vis}
\end{figure}

\begin{table}[t]
\centering
\renewcommand\arraystretch{1.2}
\resizebox{0.465\textwidth}{!}{
\begin{tabular}{l|cc|cccc}
\hline
DesNet   & DSCL        & GDG        & RMSE    & MAE    & iRMSE  & iMAE  \\ \hline
i        &             &            & 1176.4  & 336.3  & 3.7    & 1.8   \\
ii       & \checkmark  &            & 1084.5  & 312.2  & 3.3    & 1.6   \\
iii      &             & \checkmark & 1061.8  & 308.4  & 3.2    & 1.5   \\
\rowcolor[HTML]{E8E8E8}
iv       & \checkmark  & \checkmark & \textbf{969.3}   & \textbf{285.0}  & \textbf{3.0}    & \textbf{1.3}   \\ \hline
\end{tabular}
}
\caption{Ablation on DesNet with DSCL and GDG designs.}
\label{t_DesNet}
\end{table}

\begin{table}[h!]
\centering
\tiny
\renewcommand\arraystretch{1.2}
\resizebox{0.465\textwidth}{!}{
\begin{tabular}{l|cccc}
\hline
DSCL        & RMSE    & MAE    & iRMSE  & iMAE  \\ \hline
\rowcolor[HTML]{E8E8E8}
num.=1      & 1084.5  & 312.2  & 3.3    & 1.6   \\             
num.=4      & 1076.8  & 310.5  & 3.3    & 1.5   \\   
num.=8      & 1070.6  & 308.7  & 3.2    & 1.5   \\
num.=16     & \textbf{1068.2}  & \textbf{308.2}  & \textbf{3.1}    & \textbf{1.4}   \\ 
num.=HW     & 1189.4  & 343.5  & 3.5    & 1.6   \\ \hline
\end{tabular}
}
\caption{Ablation on DSCL with different numbers (num.) of elements in the scale factor matrix.}
\label{t_DSCL}
\end{table}

\noindent reference for areas of missing depth, owing to the dilation and attention designs from which the model is urged to learn prior density. \textbf{(3)} Finally, to combine the best of both worlds, we simultaneously embed DSCL and GDG (DesNet-iv) into the baseline. Consequently, our model performs much better than DesNet-i, enormously surpassing it by $207.1mm$ in RMSE, $51.3mm$ in MAE, $0.7{km}^{-1}$ in iRMSE, and $0.5{km}^{-1}$ in iMAE. Besides, comparing the 4th column with the 2nd column of Fig.~\ref{ab_vis}, it is noticeable that depth predictions of DesNet-iv clearly possess more reasonable visual effects than those of the baseline DesNet-i.

\textbf{DSCL.} The final layer of our scale network is the adaptive average pooling function. Therefore, the number of elements in the scale factor matrix can be arbitrary in theory. For example, num.=4 will lead to quartering relative depth. Then we multiply each of the quartering by the corresponding element in the scale factor matrix, finally obtaining the absolute depth prediction. Tab.~\ref{t_DSCL} reports the cases of 1, 4, 8, 16, and HW. We can find that, \textbf{(1)} as the number increases (num.$\le16$), the performance of the model gets better and better. It demonstrates that multiple region-aware scale elements is more accurate than single scale element for the full-resolution relative depth. \textbf{(2)} When num. reaches the maximum HW, the model performs even worse than the baseline DesNet-i in Tab.~\ref{t_DesNet}, which is mainly caused by the more difficult model learning. That said, two equally complex prediction targets, the full-resolution scale and relative depth, are harder for the network to learn than the single absolute depth target. Hence, the number of elements in our decomposed scale matrix is bounded in light of good performance.

\begin{table}[t]
\centering
\small
\renewcommand\arraystretch{1.2}
\resizebox{0.465\textwidth}{!}{
\begin{tabular}{l|c|cccc}
\hline
GDG-dilation   &  size  & RMSE    & MAE    & iRMSE  & iMAE  \\ \hline
bilinear       &   -    & 1162.7  & 335.1  & 3.6    & 1.8   \\
nearest        &   -    & 1148.4  & 331.0  & 3.4    & 1.7   \\ \hline
cross          &   3    & 1106.6  & 321.2  & 3.5    & 1.7   \\
square         &   3    & 1088.3  & 315.7  & 3.3    & 1.6   \\
\rowcolor[HTML]{E8E8E8}
square         &   5    & \textbf{1061.8}  & \textbf{308.4}  & \textbf{3.2}    & \textbf{1.5}   \\
square         &   7    & 1112.9  & 324.3  & 3.5    & 1.6   \\ \hline
\end{tabular}
}
\caption{Ablation on GDG with different dense depth reference produced by bilinear/nearest interpolation and morphological dilation with cross/square dilation kernels.}
\label{t_GDG_dilation}
\end{table}

\begin{table}[t]
\centering
\scriptsize
\renewcommand\arraystretch{1.2}
\resizebox{0.465\textwidth}{!}{
\begin{tabular}{l|cccc}
\hline
GDG-attention   & RMSE   & MAE      & Memory    & Time      \\ \hline
CA              & 1040.2 & 304.0    & +13.64    & +64.5     \\  
DSA             & \textbf{1024.5} & \textbf{297.4}    & +13.64    & +64.6     \\
FDSA w/o mask   & 1065.7 & 310.6    & +\textbf{4.20}     & +\textbf{12.3}     \\
\rowcolor[HTML]{E8E8E8}
FDSA w/ mask    & 1061.8 & 308.4    & +4.21     & +12.5     \\ \hline
\end{tabular}
}
\caption{Ablation on GDG with different attentions, \emph{i.e.}, the conventional attention (CA), our dense-to-sparse attention (DSA), and our fast dense-to-sparse attention (FDSA). GPU memory (G) and inference time (ms) are also considered.}
\label{t_GDG_attention}
\end{table}

\textbf{GDG-dilation.} Based on DesNet-i, Tab~\ref{t_GDG_dilation} displays the comparison of model performance using different manners to generate coarse but dense depth. On the whole, we observe that, \textbf{(1)} both interpolation and dilation methods can benefit our model since they can produce denser depth reference to compensate depth holes. Specifically, \textbf{(2)} the interpolation approach only slightly improves the baseline DesNet-i, while the dilation manner performs better. It shows that the dilation manages generating more precise dense depth than interpolation, of which can also find some evidence in~\cite{ku2018defense}. \textbf{(3)} For one thing, dilation kernels with same size but different shapes have diverse performance. Square-3 is $18.3mm$ slightly lower than that of cross-3, owing to the higher-quality depth results with denser pixels and lower error generated by the denser square dilation kernel. For another thing, dilation kernels with same shape but different sizes still have slight distinctions. The last three rows of Table~\ref{t_GDG_dilation} verifies that square-5 achieves the lowest errors among square kernels with $3\times3$, $5\times5$, and $7\times7$ sizes, which can be viewed as size-accuracy trade-off.

\begin{figure*}[t]
\centering
\includegraphics[width=2.03\columnwidth]{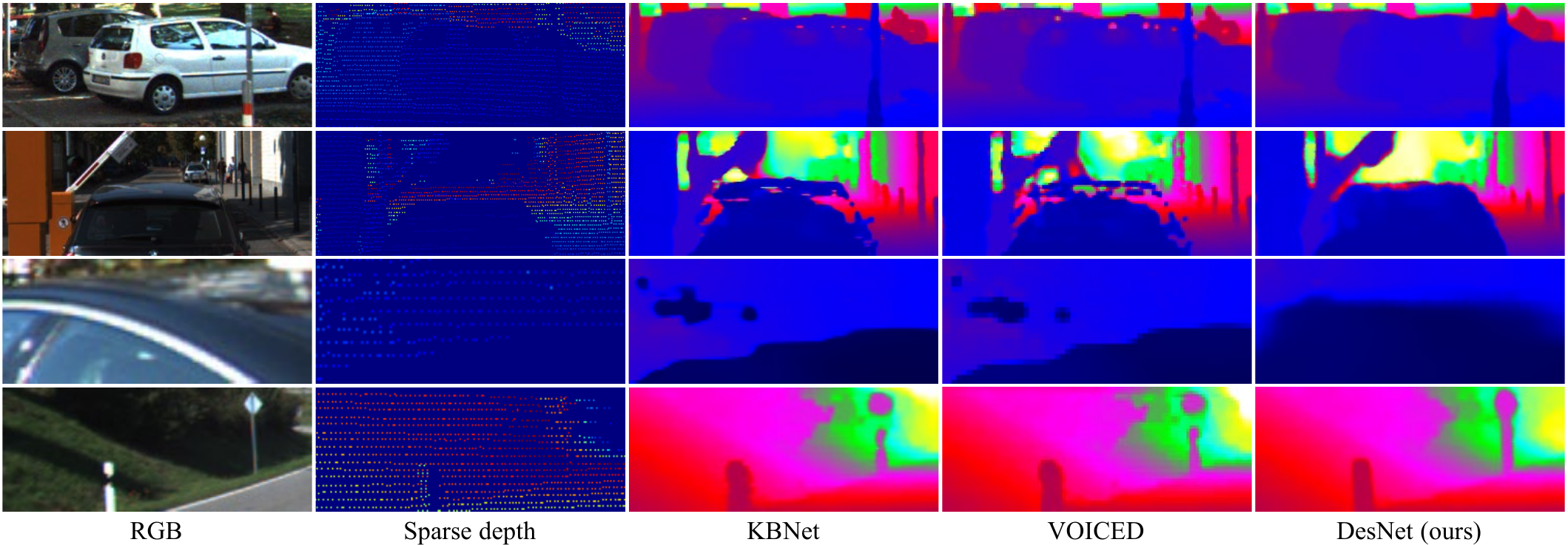}\\
\caption{Visual comparison on KITTI online depth completion benchmark, where warmer color refers to longer distance.}\label{kitti_vis}
\end{figure*}

\begin{table}[t]
\centering
\renewcommand\arraystretch{1.2}
\resizebox{0.465\textwidth}{!}{
\begin{tabular}{l|ccccc}
\hline
Method              & RMSE     & MAE     & iRMSE & iMAE  & \#P  \\ \hline
S2D                 & 1299.85  & 350.32  & 4.07  & 1.57  & 27.8 \\
IP-Basic            & 1288.46  & 302.60  & 3.78  & 1.29  & 0.0  \\
DFuseNet            & 1206.66  & 429.93  & 3.62  & 1.79  & n/a  \\
DDP$^{\ast}$        & 1263.19  & 343.46  & 3.58  & 1.32  & 18.8 \\
VOICED              & 1169.97  & 299.41  & 3.56  & 1.20  & 9.7  \\
AdaFrame            & 1125.67  & 291.62  & 3.32  & 1.16  & 6.4  \\
ScaffNet$^{\ast}$   & 1121.93  & 280.76  & 3.30  & 1.15  & 7.8  \\
SynthProj$^{\ast}$  & 1095.26  & 280.42  & 3.53  & 1.19  & 2.6  \\
KBNet               & \underline{1068.07}  & \textbf{258.36} & \underline{3.01} & \textbf{1.03} & 6.9 \\
DesNet              & \textbf{938.45}  & \underline{266.24}  & \textbf{2.95}  & \underline{1.13} & 13.4 \\ \hline
\end{tabular}
}
\caption{Quantitative results on KITTI test set. $\ast$ denotes extra synthetic data and \#P refers to model parameters (M).
}
\label{t_KITTI}
\end{table}

\textbf{GDG-attention.} Based on square-5, Tab~\ref{t_GDG_attention} validates different attention mechanisms. Overall, we discover that, \textbf{(1)} CA~\cite{dosovitskiy2020image}, our DSA, and our FDSA consistently have positive impacts on error metrics since they effectively propagate valid dense depth reference into sparse targets. \textbf{(2)} Our DSA is superior to all three others in terms of RMSE and MAE. With similar complexity, DSA surpasses CA by $15.7mm$ in RMSE and $6.6mm$ in MAE, showing that the dense-modal and sparse-modal fusion in DSA is more reasonable than the simple addition in CA. \textbf{(3)} Compared with CA, With acceptable performance degradation, \emph{i.e.}, averagely $23.55mm$ in RMSE and $5.5mm$ in MAE, our FDSA without mask largely reduces GPU memory by 9.44G, and accelerate the inference speed from $64.6ms$ to $12.3ms$. Therefore, our FDSA design is GPU-friendly. Besides, FDSA with mask slightly outperforms FDSA without mask, demonstrating the robustness of our mask strategy.

\subsection{Comparison with SoTA Methods}
Here, we compare our DesNet with existing state-of-the-art (SoTA) methods on KITTI and NYUv2 datasets, including S2D, IP-Basic, DFuseNet, and Alex Wong's series of works.

\textbf{On outdoor KITTI}, as shown in Tab.~\ref{t_KITTI}, by combining DSCL and GDG designs, DesNet achieves the lowest RMSE \& iRMSE and competitive MAE \& iMAE among all mentioned approaches. Especially in RMSE, our DesNet surpasses the best KBNet by a large margin $129.62mm$, while KBNet outperforms the second best SynthProj$^{\ast}$ only by $27.19mm$. Also, DesNet obtains competitive results in MAE and iMAE, ranking second. As we know that, RMSE is very sensitive to large depth value while MAE is more sensitive to small one. Thus, \textbf{(i)} the lowest RMSE denotes that DesNet can predict more accurate depth in long-range region. \textbf{(ii)} Worse MAE indicates that DesNet is not very good enough at recovering precise depth in close-range region. \textbf{(iii)} However, as shown in Figs.~\ref{comparison_with_KBNet} and \ref{kitti_vis}, cars in close-range region predicted by DesNet are \emph{much denser} than others. Our recovery has more reasonable visual effect. Besides, the semi-dense (about 30\%) GT depth also lacks many valid points, where the pixels are ignored when computing errors, obscuring the merits of DesNet in terms of evaluation metrics. To further validate the generalization of our method, we conduct comparative experiment \textbf{on indoor NYUv2}. Tab.~\ref{t_NYUv2} demonstrates that our DesNet achieves outstanding performance as well, which is, \emph{e.g.}, $9.51mm$ superior to the best KBNet in RMSE. In short, these evidences confirm that our DesNet actually possesses strong and robust performance.

\begin{table}[t]
\centering
\tiny
\renewcommand\arraystretch{1.2}
\resizebox{0.462\textwidth}{!}{
\begin{tabular}{l|cccc}
\hline
Method        & RMSE     & MAE     & iRMSE  & iMAE    \\ \hline
SynthProj     & 235.64   & 134.62  & 57.13  & 29.84   \\
VOICED        & 228.38   & 127.61  & 54.70  & 28.89   \\
ScaffNet      & 199.31   & 117.49  & 44.06  & 24.89   \\
KBNet         & \underline{197.77} & \underline{105.76} & \underline{42.74}  & \textbf{21.37}  \\
DesNet        & \textbf{188.26}  & \textbf{103.42}  & \textbf{38.57}  & \underline{21.44} \\ \hline
\end{tabular}
}
\caption{Quantitative results on NYUv2 official test split.}
\label{t_NYUv2}
\end{table}

\section{Conclusion}
In this paper, we proposed DesNet to utilize both the real scale information in sparse depth and the high accuracy of scale-agnostic counterpart for unsupervised depth completion, which decomposed the learning of absolute depth into relative depth prediction and global scale estimation. Such explicit learning does bring benefit. Further, to tackle the issue of depth holes, we introduced the global depth guidance to produce denser depth reference and attentively propagate it into the sparse target, severally using morphological dilation and dense-to-sparse attention. Owing to these designs, DesNet is remarkably superior to existing SoTA approaches.

\section{Acknowledgement}
The authors would like to thank all reviewers for their instructive comments. This work was supported by the National Science Fund of China under Grant Nos.~U1713208 and 62072242. Note that the PCA Lab is associated with, Key Lab of Intelligent Perception and Systems for High-Dimensional Information of Ministry of Education, and Jiangsu Key Lab of Image and Video Understanding for Social Security, School of Computer Science and Engineering, Nanjing University of Science and Technology.

\bibliography{aaai23}
\end{document}